\def\year{2019}\relax
\documentclass[letterpaper]{article}
\usepackage{aaai19}
\usepackage{times}
\usepackage{helvet}
\usepackage{courier}
\usepackage{url}            
\usepackage{graphicx}
\usepackage{booktabs}       
\usepackage{amsfonts}       
\usepackage{nicefrac}       
\usepackage{microtype}      
\usepackage{bbold} 			
\usepackage{amsmath}
\usepackage{amsthm}
\usepackage{amssymb}

\DeclareMathOperator*{\diag}{diag}
\DeclareMathOperator*{\tr}{tr}
\DeclareMathOperator*{\pre}{pre}
\DeclareMathOperator*{\rec}{rec}

\newtheorem{observation}{Observation}

\title{The SpectACl of Nonconvex Clustering: A Spectral Approach to Density-Based Clustering}
\author{Sibylle Hess\\
TU Dortmund University\\
Computer Science Faculty\\
Artificial Intelligence LS VIII\\
D-44221 Dortmund, Germany\\
\url{sibylle.hess@tu-dortmund.de}
\And
Wouter Duivesteijn\\
Technische Universiteit Eindhoven\\
Faculteit Wiskunde \& Informatica\\
Data Mining Group\\
Eindhoven, the Netherlands\\
\url{w.duivesteijn@tue.nl}
\And
Philipp Honysz\and Katharina Morik\\
TU Dortmund University\\
Computer Science Faculty\\
Artificial Intelligence LS VIII\\
D-44221 Dortmund, Germany\\
\url{philipp.honysz@tu-dortmund.de}\\
\url{katharina.morik@tu-dortmund.de}
}

\nocopyright
\begin{document}
\maketitle
\begin{abstract}
When it comes to clustering nonconvex shapes, two paradigms are used to find the most suitable clustering: minimum cut and maximum density. The most popular algorithms incorporating these paradigms are Spectral Clustering and DBSCAN. Both paradigms have their pros and cons. While minimum cut clusterings are sensitive to noise, density-based clusterings have trouble handling clusters with varying densities. In this paper, we propose \textsc{SpectACl}: a method combining the advantages of both approaches, while solving the two mentioned drawbacks. Our method is easy to implement, such as Spectral Clustering, and theoretically founded to optimize a proposed density criterion of clusterings. Through experiments on synthetic and real-world data, we demonstrate that our approach provides robust and reliable clusterings. 
\end{abstract}

\section{Introduction}
Despite being one of the core tasks of data mining, and despite having been around since the $1930$s \cite{1932Driver,1935Klimek,1939Tryon}, the question of clustering has not yet been answered in a
manner that doesn't come with innate disadvantages.  The paper that you are currently reading will also
not provide such an answer.  Several advanced solutions to the clustering problem have become quite 
famous, and justly so, for delivering insight in data where the clusters do not offer themselves up
easily.  \emph{Spectral Clustering} \cite{2004Dhillon} provides an answer to the curse of dimensionality
inherent in the clustering task formulation, by reducing dimensionality through the spectrum of the
similarity matrix of the data.  \emph{DBSCAN} \cite{1996Ester} is a density-based clustering algorithm 
which has won the SIGKDD test of time award in 2014.  Both Spectral Clustering and DBSCAN can find 
non-linearly separable clusters, which trips up naive clustering approaches; these algorithms deliver 
good results. In this paper, we propose a new clustering model which encompasses the strengths of both 
Spectral Clustering and DBSCAN; the combination can overcome some of the innate disadvantages of both 
individual methods.

\begin{figure}[t]
\centering
\includegraphics[width=.95\columnwidth]{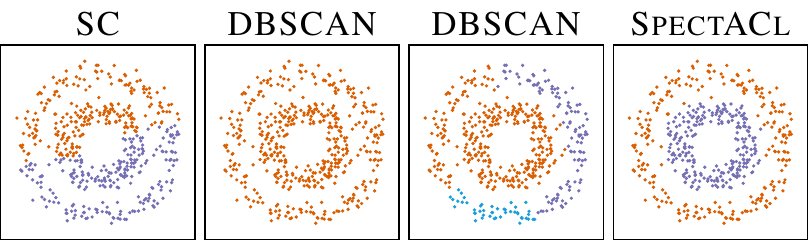}
\caption{Performance of Spectral Clustering, DBSCAN, and \textsc{SpectACl} on two concentric circles. Best viewed in color.}
\label{fig:intro}
\end{figure}

For all their strengths, even the most advanced clustering methods nowadays still can be tripped up by some pathological cases: datasets where the human observer immediately sees what is going on, but which prove to remain tricky for all state-of-the-art clustering algorithms.  One such example is the dataset illustrated in Figure \ref{fig:intro}: we will refer to it as the \emph{two circles} dataset.  It consists of two noisily\footnote{the scikit-learn clustering documentation (cf. \url{https://scikit-learn.org/stable/modules/clustering.html}) shows how \textsc{SC} and DBSCAN can succeed on this dataset, but that version contains barely any noise (scikit's noise parameter set to 0.05, instead of our still benign 0.1).} separated concentric circles, each encompassing the same number of observations.  As the leftmost plot in Figure \ref{fig:intro} shows, Spectral Clustering does not at all uncover the innate structure of the data: two clusters are discovered, but both incorporate about half of each circle.  It is well-known that Spectral Clustering is highly sensitive to noise \cite[Figure 1]{2017Bojchevski}; if
two densely connected communities are additionally connected to each other via a narrow bridge of only
a few observations, Spectral Clustering runs the risk of reporting these communities plus the bridge as
a single cluster, whereas two clusters plus a few noise observations (outliers) would be the desired 
outcome.
The middle plots in Figure \ref{fig:intro} shows how DBSCAN (with $minpts$ set to $25$ and $26$, respectively) fails to realize that there are two clusters. This is hardly surprising, since DBSCAN is known to struggle with several clusters of varying density, and that is exactly what we are dealing with here: since both circles consist of the same number of observations, the inner circle is substantially more dense than the outer circle.  The rightmost plot in Figure \ref{fig:intro} displays the result of the new clustering method that we introduce in this paper, \textsc{SpectACl} (\emph{Spectral Averagely-dense Clustering}): it accurately delivers the clustering that represents the underlying phenomena in the dataset.

\subsection{Main Contributions}

In this paper, we provide \textsc{SpectACl}, a new clustering method which combines the benefits of Spectral Clustering and
DBSCAN, while alleviating some of the innate disadvantages of each individual method.  
Our method finds clusters having a large average density, where the appropriate density for each cluster is automatically determined through the spectrum of the weighted adjacency matrix.
Hence, \textsc{SpectACl} does not suffer
from sensitivity to the $\text{minPts}$ parameter as DBSCAN does, and unlike DBSCAN it can natively handle several clusters with varying densities.  As in the Spectral Clustering pipeline, the final step of \textsc{SpectACl} is an embedding postprocessing step using $k$-means. However, unlike Spectral Clustering, we demonstrate the fundamental soundness of applying $k$-means to the embedding step in \textsc{SpectACl}: from \textsc{SpectACl}'s objective function we derive an upper bound by means of the eigenvector decomposition; we derive that the optimization of our upper bound is equal to $k$-means on the eigenvectors.  Our Python implementation, and the data generating and evaluation script, are publicly available\footnote{\url{https://sfb876.tu-dortmund.de/spectacl}}.
\section{A Short Story of Clustering}
A formal discussion of the state of clustering requires some notation. We assume that our data is given by the matrix $D\in\mathbb{R}^{m\times n}$. The data represents $m$ points $D_{j\cdot}$ for $1\leq j\leq m$ in the $n$-dimensional feature space. For every point, we denote with $\mathcal{N}_\epsilon(j)=\{l| \|D_{j\cdot}-D_{l\cdot}\|<\epsilon\}$ its $\epsilon$-neighborhood.

Given an arbitrary matrix $M$, we denote with $\|M\|$ its Frobenius norm. If $M$ is diagonalizable, then we denote its truncated eigendecomposition of rank $d$ as $M\approx V^{(d)}\Lambda^{(d)} {V^{(d)}}^\top$. The matrix $\Lambda$ is a $d\times d$ diagonal matrix, having the eigenvalues of $M$ on its diagonal, which have the $d$-largest absolute values: $|\lambda_1|=|\Lambda_{11}|\geq \ldots \geq |\Lambda_{dd}|$. The matrix $V^{(d)}$ concatenates the corresponding eigenvectors.

We write $\mathbf{1}$ to represent a matrix having all elements equal to 1. The dimension of such a matrix can always be inferred from the context. We also use the shorthand $\mathbb{1}^{m\times l}$ for the set of all binary partition matrices. That is, $M\in\mathbb{1}^{m\times l}$ if and only if $M\in\{0,1\}^{m\times l}$ and $|M_{j\cdot}|=1$ for all $j$. 
\subsection{$k$-means}
If there is one algorithm which comes to mind when thinking about clustering, it is likely the $k$-means algorithm \cite{1982Lloyd}. The popularity of $k$-means is due to its simple and effective yet theoretically founded optimization by alternating minimization, and the intuitively easily graspable notion of clusters by the within-cluster point scatter.  
The underlying assumption for this kind of clustering is that points in one cluster are similar to each other on average. As fallout, this imposes the undesired constraint that clusters are separated by a Voronoi tesselation and thus, have convex shapes.

We denote the objective of $k$-means in its matrix factorization variant:
\begin{align}\label{eq:kmeansNorm}
\min_{Y,X} \|D-YX^\top\|^2 \text{ s.t. } Y\in\mathbb{1}^{m\times r}, X\in\mathbb{R}^{n\times r}.
\end{align}
Here, the matrix $Y$ indicates the cluster memberships, i.e.,  $Y_{js}=1$ if point $j$ is assigned to cluster $s$ and $Y_{js}=0$ otherwise. The matrix $X$ represents the cluster centroids. The objective is nonconvex, but convex if one of the matrices is fixed. The optimum with respect to the matrix $X$ for fixed cluster assignments $Y$ is attained at $X=D^\top Y (Y^\top Y)^{-1}$. This is a matrix formulation of the simple notion that each cluster centroid is equal to the average of points in the cluster~\cite{2015Bauckhage}.

The objective of Equation~\eqref{eq:kmeansNorm} is equivalent to a trace maximization problem~\cite{2006Ding}, which derives from the definition of the Frobenius norm by means of the trace
\begin{align}\label{eq:kmeansTr}
    \max_Y \tr(Y^\top DD^\top Y(Y^\top Y)^{-1}) \text{ s.t. } Y\in\mathbb{1}^{m\times r}.
\end{align}
Note that this objective formulation of $k$-means depends only on the similarities of data points, expressed by the inner product of points. This is a stepping stone to the application of kernel methods and the derivation of nonconvex clusters.
\subsection{Graph Cuts}
The restriction to convex clusters can be circumvented by a transformation of the data points into a potentially higher-dimensional Hilbert space, reflected by a kernel matrix \cite{1998Schoelkopf}.
Such a representation based only on the similarity of points introduces an interpretation of the data as a graph, where each data point corresponds to a node, and the denoted similarities are the weights of the corresponding edges. In this view, a cluster is identifiable as a densely connected component which has only weak connections to points outside of the cluster. In other words, if we cut the edges connecting the cluster to its outside, we strive to cut as few edges as possible.  
This is known as the minimum cut problem. Given a possibly normalized edge weight matrix $W$, the cut is: 
\begin{align}\label{eq:MinCut}
    \text{cut}(Y;W) = \sum_{s=1}^r Y_{\cdot s}^\top W(\mathbf{1}-Y_{\cdot s}).
\end{align}
However, minimizing the cut often returns unfavorable, unbalanced clusterings, where some clusters encompass only a few or single points. Therefore, normalizing constraints have been introduced.
One of the most popular objective functions incorporating this normalization is the ratio cut~\cite{1992Hagen}. The objective of ratio cut is given as:
\begin{align}
\text{RCut}(Y;W) = \sum_{s=1}^r \frac{Y_{\cdot s}^\top W(\mathbf{1}-Y_{\cdot s})}{|Y_{\cdot s}|}.\label{eq:RCut}
\end{align}
Minimizing the ratio cut is equivalent to a trace maximization problem~\cite{2004Dhillon,2005Ding}, given as:
\begin{align}\label{eq:traceminCut}
 \max_{Y}&\ \tr(Z^\top (-L)Z) \\
 \text{ s.t. }& Z=Y(Y^\top Y)^{-1/2},Y\in\mathbb{1}^{m\times r}, \nonumber 
\end{align}
where $L=\diag(W\mathbf{1}) - W$ is known as the difference Laplacian~\cite{1997Chung}. Note that a normalization of $W$ would change the matrix $L$ to another famous graph Laplacian: the symmetric or random walk Laplacian~\cite{1997Chung}. 

The matrix $-L$ is negative semi-definite, having eigenvalues $0=\lambda_1\geq \ldots\geq \lambda_m$.
We observe that substituting the matrix $-L$ from Equation~\eqref{eq:traceminCut} with the positive semi-definite matrix $|\lambda_m|I-L$ does not change the objective. Therefore, ratio cut is equivalent to kernel $k$-means, employing the kernel matrix $|\lambda_m|I-L$.

\subsection{Spectral Clustering}
According to the Ky Fan theorem~\cite{1949Fan}, the maximum of Equation~\eqref{eq:traceminCut} with respect to the matrix $Z\in\mathbb{R}^{m\times r}$, having orthogonal columns, is attained at the eigenvectors to the $r$ largest eigenvalues of f $|\lambda_m|I-L$, i.e., $Z=V^{(r)}$. These relaxed results then need to be discretized in order to obtain a binary cluster indicator matrix $Y$. Most often, $k$-means clustering is applied for that purpose. 

The eigenvectors of eigenvalue zero from the Laplacian indicate connected components, which supports a theoretic justification of the embedding given by the eigenvectors $V^{(r)}$. Hence, Spectral Clustering is backed up by the thoroughly studied theory of graph Laplacians.
Spectral Clustering is in general described by the following three steps:
\begin{enumerate}
\item choose the representation of similarities by the weight matrix $W$, calculate the Laplacian $L=\diag(W\mathbf{1})-W$;
\item compute the truncated eigendecomposition of the translated Laplacian $|\lambda_m|I-L\approx V^{(r+1)}\Lambda^{(r+1)} {V^{(r+1)}}^\top$;
\item compute a $k$-means clustering on the eigenvectors $V^{(r+1)}$ (excluding the first eigenvector) returning $r$ clusters.
\end{enumerate}
The weighted adjacency matrix is usually generated by a $k$-nearest neighbor graph. Given $k$, one way to define the weighted adjacency matrix is $W=\nicefrac{1}{2}(A+A^\top)$, when $A$ is the adjacency matrix of the $k$-nearest neighbor graph. The other variant, which is less often used, is the determination of $W$ by the $\epsilon$-neighbor graph. In this variant, we have $W_{jl}=1$ if and only if $l\in\mathcal{N}_{\epsilon}(j)$ and $W_{jj}=0$.

It is advised to employ a normalization of the weighted adjacency matrix, e.g., the symmetric normalization $\tilde{W}=\diag(W\mathbf{1})^{-1/2}W\diag(W\mathbf{1})^{-1/2}$. The corresponding Laplacian is the symmetrically normalized Laplacian $L_{sym} = I - \tilde{W}$. For a more detailed discussion of Spectral Clustering methods, see~\cite{2007Luxburg}.

Spectral Clustering is a popular technique, since the truncated eigendecomposition is efficiently computable for large-scale sparse matrices~\cite{2011Saad}, and efficient to approximate for general matrices by randomized methods~\cite{2011Halko}. However, Spectral Clustering also has multiple shortcomings. One of the drawbacks of this pipeline of methods is the uncertainty about the optimized objective. Although there is a relation to the rational cut objective, as outlined, there are examples where the optimal solution of rational cut and the solution of Spectral Clustering diverge~\cite{1998Guattery}. Related to this problem is the somewhat arbitrary choice to employ $k$-means to binarize the relaxed result. In principle, also other discretization methods could and have been applied~\cite{2007Luxburg}, yet there is no theoretical justification for any of these. 
\subsection{The Robustness Issue}
Beyond theoretical concerns, an often reported issue with Spectral Clustering results is the noise sensitivity. This might be attributed to its relation with the minimum cut paradigm, where few edges suffice to let two clusters appear as one.   

The state-of-the-art in Spectral Clustering research includes two approaches aiming at increasing robustness towards noise. The one is to incorporate ideas from density-based clustering, such as requiring that every node in a cluster has a minimal node degree~\cite{2017Bojchevski}. Minimizing the cut subject to the minimum degree constraint resembles finding a connected component of core points, as performed by DBSCAN. The drawback of this approach is that it introduces a new parameter to Spectral Clustering, influencing the quality of the result.  
The other is to learn the graph representation simultaneously with the clustering~\cite{2017Bojchevski,2018Kang,2017Nie}. This is generally achieved by alternating updates of the clustering and the graph representation, requiring the computation of a truncated eigendecomposition in every iteration step. This results in higher computational costs.  A final alternative is the approach based on Dominant Sets, as outlined in \cite{2016Hou}.  This approach is a sequential method, iteratively deriving the single best cluster and reducing the dataset accordingly, unlike \textsc{SpectACl} which simultaneously optimizes all clusters.
Of these newer methods, we compare our method against the Robust Spectral Clustering algorithm~\cite{2017Bojchevski}, since it incorporates both the notion of density and the learning of the graph representation.
\section{Spectral Averagely-Dense Clustering}
We propose a cluster definition based on the average density (i.e.,\@ node degree) in the subgraph induced by the cluster. Let $W$ be the adjacency matrix. Since we are interested in deviations of the nodes' degree, we employ the $\epsilon$-neighborhood graph to determine $W$. We strive to solve the following problem, maximizing the \emph{average cluster density}:
\begin{align}\label{eq:avgDensObj}
\max_{Y\in \mathbb{1}^{m\times r}}\tr(Y^\top WY (Y^\top Y)^{-1})=\sum_s \delta(Y_{\cdot s},W).
\end{align}
The objective function returns the sum of average node degrees $\delta(Y_{\cdot s},W)$ in the subgraph induced by cluster $s$, i.e.:
\begin{align}\label{eq:delta}
\delta(Y_{\cdot s},W) = \frac{Y_{\cdot s}^\top W Y_{\cdot s}}{\|Y_{\cdot s}\|^2} =\frac{1}{|Y_{\cdot s}|}\sum_{j:Y_{js}=1} W_{j\cdot}Y_{\cdot s}.
\end{align}  
Note that our Objective~\eqref{eq:avgDensObj} is equivalent to minimum cut if the matrix $W$ is normalized. In this case, the Laplacian is given as $L=I-\tilde{W}$, and subtracting the identity matrix from Equation~\eqref{eq:avgDensObj} does not change the objective.

We derive a new relationship for the solution of Objective~\eqref{eq:avgDensObj} and the spectrum of the adjacency matrix. Thereby, we establish a connection with the application of $k$-means to the spectral embedding. As a result, our method encompasses the same steps as Spectral Clustering, and hence it can be efficiently computed even for large scale data. 
\subsection{Dense Clusters and Projected Eigenvectors}
The function $\delta$ from Equation~\eqref{eq:delta} is also known as the Rayleigh quotient. The values of the Rayleigh quotient 
depend on spectral properties of the applied matrix. As such, the density $\delta(y,W)\in [\lambda_m,\lambda_1]$ is bounded by the smallest and largest eigenvalues of the matrix $W$~\cite{1978Collatz}. The extremal densities are attained at the corresponding eigenvectors. A simple calculation shows that the eigenvectors $V_{\cdot 1},\ldots,V_{\cdot d}$ to the $d$-largest eigenvalues span a space whose points have a minimum density $\delta(y,W)\geq\lambda_d$.   
Thus, the projection of $W$ onto the subspace spanned by the first eigenvectors reduces the dimensionality of the space in which we have to search for optimal clusters.

 This insight suggests a naive approach, where we approximate $W$ by its truncated eigendecomposition. We restate in this case the optimization problem to find averagely dense clusters from Equation~\eqref{eq:avgDensObj} as
\begin{align}\label{eq:approxTrunc}
\max_{Y\in\mathbb{1}^{m\times r}} \sum_s \frac{Y_{\cdot s}^\top V^{(d)}\Lambda^{(d)} {V^{(d)}}^\top Y_{\cdot s}}{|Y_{\cdot s}|}.
\end{align}
A comparison with the $k$-means objective from Equation~\eqref{eq:kmeansTr} shows that the objective above is a trace maximization problem which is equivalent to $k$-means clustering on the data matrix $U=V^{(r)}\left(\Lambda^{(r)}\right)^{1/2}$. 

Unfortunately, applying $k$-means to the matrix $U$ does not yield acceptable clusterings. The objective of $k$-means is nonconvex and has multiple local solutions. The number of local solutions increases with every eigenvector which we include in the eigendecomposition from $W$, due to the orthogonality of eigenvectors. That is, with every included eigenvector, we increase the dimension of the subspace, in which we search for suitable binary cluster indicator vectors. As a result, $k$-means returns clusterings whose objective function value approximates the global minimum, but which reflect seemingly haphazard groupings of the data points.

The eigenvectors are not only orthogonal, but also real-valued, having positive as well as negative values. The first eigenvectors have a high value of the density function $\delta$, but the mixture of signs in their entries makes an interpretation as cluster indicators difficult. If the eigenvectors would be nonnegative, then they would have an interpretation as fuzzy cluster indicators, where the magnitude of the entry $V_{jk}$ indicates the degree to which point $j$ is assigned to the fuzzy cluster $k$. Applying $k$-means to a set of highly dense fuzzy cluster indicators would furthermore reduce the space in which we search for possible cluster centers to the positive quadrant. One possibility is to approximate the matrix $W$ by a symmetric nonnegative matrix factorization, but this replaces the polynomially solvable eigendecomposition with an NP-hard problem~\cite{vavasis2009complexity}. 
We conduct another approach, showing that fuzzy cluster indicators are derivable from the eigenvalues by a simple projection.

\begin{observation}\label{thm:project}
Let $W$ be a symmetric real-valued matrix, and let $v$ be an eigenvector to the eigenvalue $\lambda$. Let $v=v^+-v^-$, with $v^+,v^-\in\mathbb{R}^m_+$ be the decomposition of $v$ into its positive and negative parts. The nonnegative vector $u =v^+ + v^-$ has a density 
\[\delta(u) \geq |\lambda|.\]
\end{observation}
\begin{proof}
Application of the triangle inequality shows that the entries of an eigenvector having a high absolute value, play an important role for achieving a high Rayleigh quotient. Let $v$ be an eigenvector to the eigenvalue $\lambda$ of $W$, then:
\begin{align}
|\lambda| |v_j| = \left|W_{j\cdot}v\right| = \left|\sum_{l}W_{jl}v_l\right| \leq 
\sum_{l}W_{jl}|v_l|. \label{eq:absEigen}
\end{align} 
Since $u_j=|v_j|$, from the inequality above follows that $Wu\geq\lambda u$. Multiplying the vector $u^\top$ from the left and dividing by $\|u\|^2$ yields the stated result.
\end{proof}
\begin{figure}[t]\centering
\includegraphics[width=.95\columnwidth]{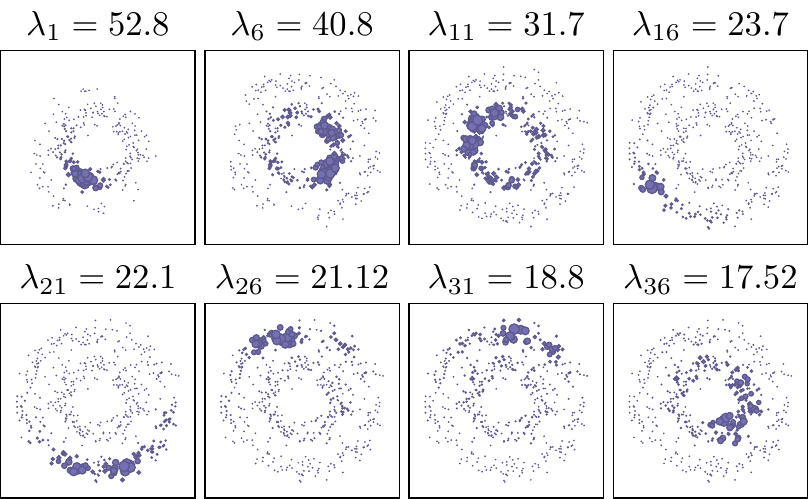}
\caption{Visualization of every fifth eigenvector for the two circles dataset. The size of each point is proportional to the absolute value of the corresponding entry in the eigenvector.}
\label{fig:binaryEmb}
\end{figure}
We refer to the eigenvectors, whose entries are replaced with their absolute values, i.e., $u_j=|v_j|$ for $1\leq j\leq m$, as \emph{projected eigenvectors}. 
The projected eigenvectors have locally similar values, resulting in the gradually changing shapes, illustrated on the two circles dataset in Figure~\ref{fig:binaryEmb}. We see that the projected eigenvectors have an interpretation as fuzzy cluster indicators: a strongly indicated dense part of one of the two clusters fades out along the circle trajectory. Furthermore, we see that the projected eigenvectors are not orthogonal, since some regions are repeatedly indicated over multiple eigenvectors. These multiple views on possible dense and fuzzy clusters might be beneficial in order to robustly identify the clustering structure.   

\begin{figure*}[t]
\centering
\includegraphics[width=.95\textwidth]{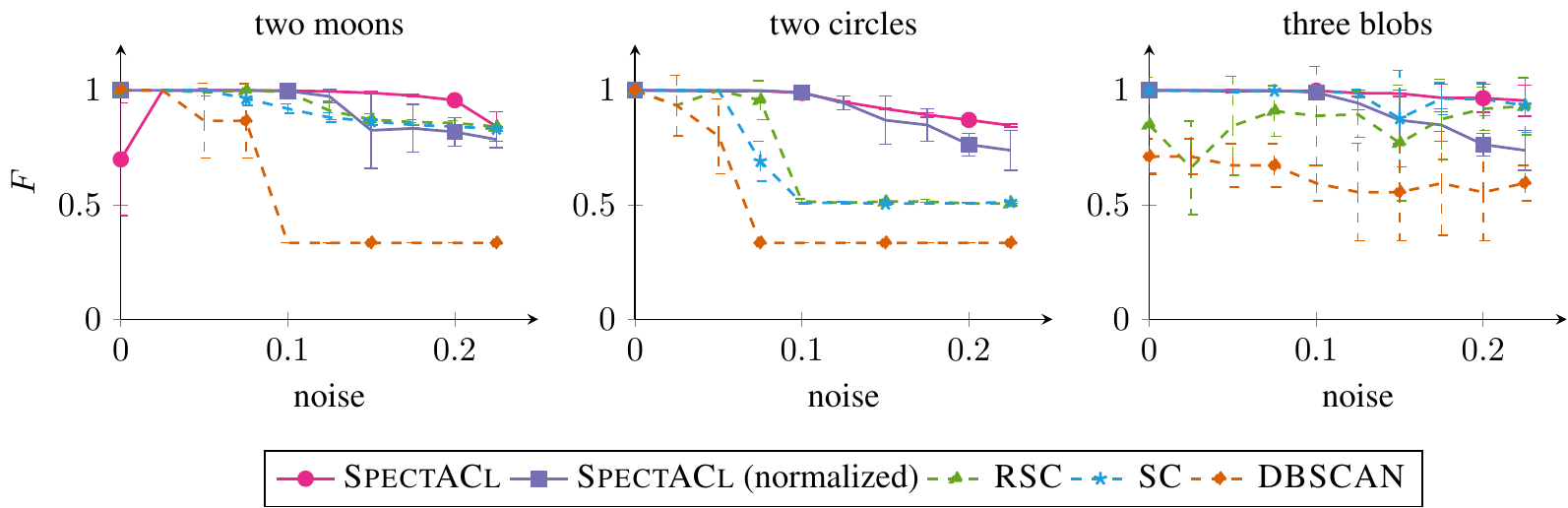}
\caption{Variation of noise, comparison of $F$-measures (the higher the better) for the two moons (left), the two circles (middle) and three blobs (right) datasets.}
\label{fig:noisePlot}
\end{figure*}
Applying $k$-means to the first $d$ projected eigenvectors yields a set of $r$ binary basis vectors, which indicate clusters in the subspace of points with an average density larger than $\lambda_d$. We summarize the resulting method \textsc{SpectACl} (Spectral Averagely-dense Clustering) with the following steps:
\begin{enumerate}
    \item compute the adjacency matrix $W$;
    \item compute the truncated eigendecomposition\\ $W\approx V^{(d)}\Lambda^{(d)}{V^{(d)}}^\top$;
    \item compute the projected embedding $U_{jk}=|V_{jk}^{(d)}||\lambda_k|^{1/2}$;
    \item compute a $k$-means clustering, finding $r$ clusters on the embedded data $U$. 
\end{enumerate}
These steps perform a minimization of an upper bound on the average cluster density function with respect to $W$ if $d=m$. In the general case $d\ll m$, we optimize an approximation of the objective function, replacing $W$ with the nonnegative product $UU^\top$. Certainly, a similar argument could be made for the pipeline of Spectral Clustering, considering the approximate objective~\eqref{eq:approxTrunc} for $d=r$. However, the difference is that the result of Spectral Clustering deteriorates with an increasing accuracy of the eigendecomposition approximation $(d\rightarrow m)$ due to orthogonality and negative entries in the eigenvectors as outlined above. We show in our experimental evaluation that the result of our method does in average not deteriorate with an increasing dimension of the embedding.

The matrix $W$ from the first step is here calculated by the $\epsilon$-neighborhood graph. However, our method is in principle applicable to any provided (weighted) adjacency matrix. We recall that the average density objective is equivalent to the ratio cut objective if we normalize the matrix $W$. We refer to this version as normalized \textsc{SpectACl}. In the normalized case, we compute the adjacency matrix according to the $k$-nearest neighbor graph, which is suggested for applications of Spectral Clustering.

Hence, our method only depends on the parameters $r$ and $\epsilon$, respectively $k$. How to determine the number of clusters, which is a general problem for clustering tasks, is beyond scope of this paper. We do provide a heuristic to determine $\epsilon$ and evaluate the sensitivity to the parameters $\epsilon$ and $k$. 
\section{Experimental Evaluation}
We conduct experiments on a series of synthetic datasets, exploring the ability to detect the ground truth clustering in the presence of noise. On real-world data, we compare the Normalized Mutual Information (NMI) of obtained models with respect to provided classes. A small qualitative evaluation is given by means of a handwriting dataset, illustrating the clusters obtained by \textsc{SpectACl}. 
\subsection{Synthetic Datasets}
We generate benchmark datasets, using the renowned scikit library. For each shape ---moons, circles, and blobs--- and noise specification we generate $m=1500$ data points. The noise is Gaussian, as provided by the scikit noise parameter; cf.\@ \url{http://scikit-learn.org}. Our experiments pit five competitors against each other: two new ones presented in this paper, and three baselines.  
We compare our method \textsc{SpectACl}, using the $\epsilon$-neighborhood graph represented by $W$ (denoted \textsc{SpectACl}) and the symmetrically normalized adjacency matrix $\diag(W_{knn}\mathbf{1})^{-1/2}W_{knn}\diag(W_{knn}\mathbf{1})^{-1/2}$ where $W_{knn}$ is the adjacency matrix to the $k$-nearest neighbor graph (denoted \textsc{SpectACl} (Normalized), or (N) as shorthand). The parameter $\epsilon$ is determined as the smallest radius such that $90\%$ of the points have at least ten neighbors and the number of nearest neighbors is set to $k=10$. We compare against the scikit implementations of DBSCAN (denoted DBSCAN) and symmetrically normalized Spectral Clustering (denoted \textsc{SC}). For DBSCAN, we use the same $\epsilon$ as for \textsc{SpectACl} and set the parameter $minPts=10$, which delivered the best performance on average. We also compare against the provided Python implementation of Robust Spectral Clustering~\cite{2017Bojchevski} (denoted \textsc{RSC}), where default values apply. Unless mentioned otherwise, our setting for the embedding dimensionality is $d=50$.

\subsubsection{Evaluation}
We assess the quality of computed clusterings by means of the ground truth. Given a clustering $Y\in\{0,1\}^{m\times r}$ and the ground truth model $Y^*\in\{0,1\}^{m\times r}$, we compute the $F$-measure. The bijection of matching clusters $\sigma$ is computed with the Hungarian algorithm~\cite{1955Kuhn}, maximizing $F=\sum_s F(s,\sigma(s))$, where
\[F(s,t) = 2\frac{\pre(s,t)\rec(s,t)}{\pre(s,t)+\rec(s,t)},\]
$$
\pre(s,t) = \frac{Y_{\cdot s}^\top Y^*_{\cdot t}}{|Y_{\cdot s}|} \quad\text{and}\quad \rec(s,t) = \frac{Y_{\cdot s}^\top Y^*_{\cdot t}}{|Y^*_{\cdot t}|}.
$$
These functions denote precision and recall, respectively. The $F$-measure takes values in $[0,1]$; the closer to one, the more similar are the computed clusterings to the ground truth. 
\subsubsection{Noise Sensitivity}
\begin{figure}[t]
\centering
\includegraphics[width=.95\columnwidth]{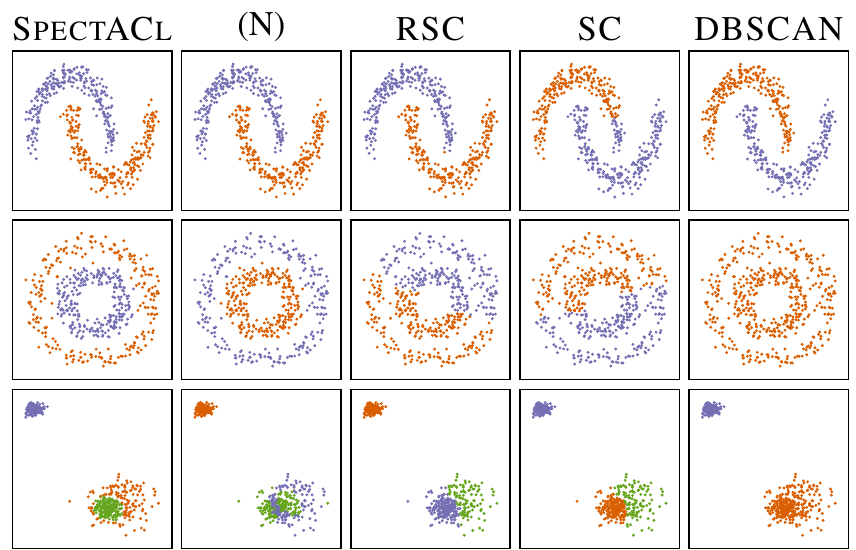}
\caption{Cluster visualizations of regarded algorithms and synthetic datasets. Best viewed in color.}
\label{fig:synthViz}
\end{figure}
In this series of experiments, we evaluate the robustness of considered algorithms against noise. For every noise and shape specification, we generate five datasets. A visualization of the clusterings for three generated datasets is given in Figure~\ref{fig:synthViz}, setting the noise parameter to $0.1$. Note that the depicted plot of the three blobs dataset shows a rare but interesting case, where two of the clusters are overlapping. In Figure~\ref{fig:noisePlot} we plot for every considered algorithm the averaged $F$-measures against the noise.  

From Figure~\ref{fig:noisePlot} we observe that \textsc{SpectACl} typically attains the highest $F$-value, showing also a very low variance in the results. The sole exception is the case of the two moons dataset when no noise is added (leftmost figure). A closer inspection of the embedding in this case shows, that \textsc{SpectACl} is tripped up by the absolute symmetry in the synthetic data. Since this is unlikely to occur in the real world, we do not further elaborate on this special case. 
The normalized version of \textsc{SpectACl} does not exhibit the artifact of symmetry, yet its results are less accurate when the noise is high ($>0.1$). Particularly interesting is the comparison of normalized \textsc{SpectACl} with \textsc{SC}, since both methods employ the same weighted adjacency matrix. We observe that both methods have a similar $F$-measure for the two moons as well as the three blobs data when the noise is at most $0.15$. For the three blobs data and a higher noise level, normalized \textsc{SpectACl} has a lower $F$-value than \textsc{SC}. However, remarkable is the difference at the two circles dataset, where  normalized \textsc{SpectACl} attains notably higher $F$-values than \textsc{SC}. As Figure \ref{fig:synthViz} illustrates, normalized \textsc{SpectACl} detects the trajectory of the two circles while \textsc{SC} cuts both circles in half. 

Comparing these results with Robust Spectral Clustering (\textsc{RSC}), learning the graph structure together with the Spectral Clustering does not drastically affect the result. Given the two moons or the two circles dataset, the $F$-measure of \textsc{RSC} is close to that of \textsc{SC}. The example clusterings in the top row of Figure~\ref{fig:synthViz} show that \textsc{RSC} is able to correct the vanilla Spectral Clustering, where a chunk from the wrong moon is added to the blue cluster. Yet, in the case of the two circles with different densities, both methods fail to recognize the circles (cf.\@ Figure~\ref{fig:synthViz}, middle row). Surprisingly, on the easiest-to-cluster dataset of the three blobs, \textsc{RSC} attains notably lower $F$-measures than Spectral Clustering.

The plots for DBSCAN present the difficulty to determine a suitable setting of parameters $\epsilon$ and $minpts$. The method to determine these parameters is suitable if the noise is low. However, beyond a threshold, DBSCAN decides to pack all points into one cluster. We also adapted the heuristic for the parameter selection, determining $\epsilon$ as the minimum radius such that $70-90\%$ of all points have at least $minPts=10$ neighbors. This does not fundamentally solve the problem: it merely relocates it to another noise threshold.
\subsubsection{Parameter Sensitivity}
We present effects of the parameter settings for \textsc{SpectACl}, namely the embedding dimension $d$, the neighborhood radius $\epsilon$ and the number of nearest neighbors $k$. We summarize the depiction of these experiments in Figure \ref{fig:paramPlot}. The plots on the top row show that the results of \textsc{SpectACl} are robust to the parameters determining the adjacency matrix. The fluctuations of the $F$-measure when $\epsilon>0.6$ are understandable, since the coordinates of points are limited to $[0,3]$. That is, $\epsilon=0.6$ roughly equates the diameter of the smaller blob on the top left in Figure~\ref{fig:synthViz} and is quite large. 
\begin{figure}[t]
\centering
\includegraphics[width=.95\columnwidth]{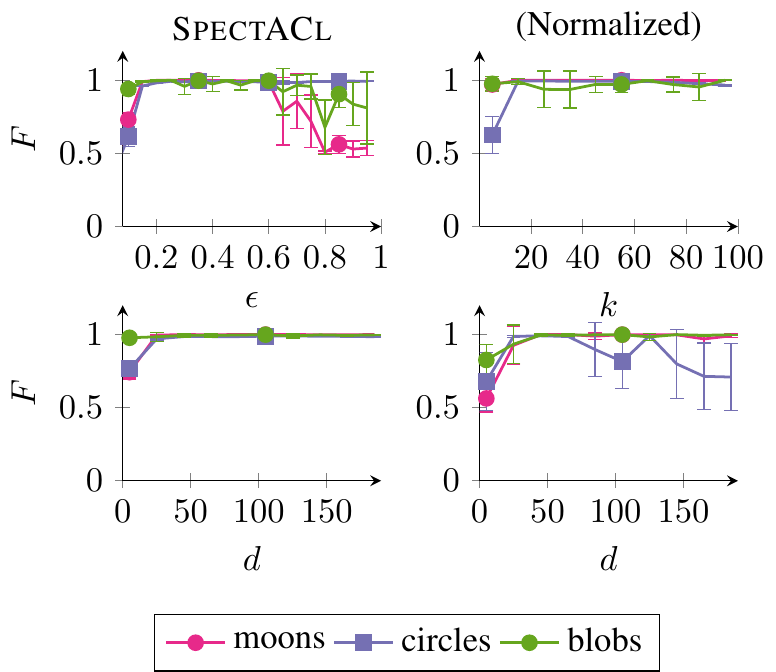}
\caption{Variation of latent parameters used for \textsc{SpectACl}. The neighborhood radius $\epsilon$ (top left) and the number of nearest neighbors $k$ (top right) determine the adjacency matrix for \textsc{SpectACl} and its normalized version. The parameter $d$ (bottom left and right) specifies the number of projected eigenvectors. We plot the $F$-measure (the higher the better) against the variation of the parameter.}
\label{fig:paramPlot}
\end{figure}
The plots on the bottom of Figure~\ref{fig:paramPlot} indicate that \textsc{SpectACl} in the unnormalized and normalized version generally only require the dimension of the embedding to be large enough. In this case, setting $d>25$ is favorably. The only exception is the two circles dataset when using normalized \textsc{SpectACl}, where the $F$-measure starts to fluctuate if $d>75$. A possible explanation for this behavior is that eigenvalues in the normalized case lie between zero and one. When we inspect the first $50$ projected eigenvectors of the normalized adjacency matrix, then we see that the last 19 eigenvectors do not reflect any sensible cluster structure, having eigenvalues in $[0.4,0.54]$. In unnormalized \textsc{SpectACl}, the unhelpful eigenvectors have a very low density and thus also a very low eigenvalue, being barely taken into account during the optimization.

\subsection{Real-World Datasets}
\begin{table}
	\centering
	\begin{tabular}{lrrrr}\toprule
	Data & m & n & r & $\mu\left(\frac{\mathcal{N}_\epsilon(j)}{m}\right)$ \\ \midrule
    Pulsar & 17\thinspace898 & 9 & 2 & $0.11\pm 0.08$  \\
    Sloan & 10\thinspace000 & 16 & 3 & $0.09\pm 0.07$\\
    MNIST & 60\thinspace000 & 784 & 10 & $0.05\pm 0.06$\\
    SNAP & 1\thinspace005 & -- & 42 & $0.05\pm 0.06$\\ \bottomrule
    \end{tabular}
\caption{Real-world dataset characteristics; number of samples $m$, features $n$, clusters (classes) $r$ and relative average number of neighbors in the $\epsilon$-neighborhood graph $\mu$.}
\label{tbl:realDescr}
\end{table}
We conduct experiments on selected real-world datasets, whose characteristics are summarized in Table~\ref{tbl:realDescr}. The Pulsar dataset\footnote{\url{https://www.kaggle.com/pavanraj159/predicting-pulsar-star-in-the-universe}} contains samples of Pulsar candidates, where the positive class of real Pulsar examples poses a minority against noise effects. The Sloan dataset\footnote{\url{https://www.kaggle.com/lucidlenn/sloan-digital-sky-survey}} comprises measurements of the Sloan Digital Sky Survey, where every observation belongs either to a star, a galaxy or a quasar. The MNIST dataset~\cite{1998Lecun} 
is a well-known collection of handwritten ciphers. The SNAP dataset refers to the Email EU core network data \cite{2014Leskovec}, 
which consists of an adjacency matrix (hence this dataset has no features in the traditional sense). For this dataset, we only compare the two versions of \textsc{SpectACl} and Spectral Clustering, since Robust Spectral Clustering and DBSCAN do not support a direct specification of the adjacency matrix.  
\begin{table}
	\centering
	\begin{tabular}{lrrrr}\toprule
	Algorithm & Pulsar & Sloan & MNIST & SNAP  \\\midrule
    \textsc{SpectACl} &	0.151 & 0.224 & 0.339 & 0.232\\
    \textsc{SpectACl}(N) & 0.005 & 0.071 & 0.756 & 0.104\\
    DBSCAN & 0.001 & 0.320 & 0.005 & --\\
    \textsc{SC} & 0.025 & 0.098 & 0.753 & 0.240\\
    \textsc{RSC} & 0.026 & 0.027 &0.740 & --\\ \bottomrule
    \end{tabular}
\caption{Normalized Mutual Information (NMI) score (the higher the better) for real-world datasets.}
\label{tbl:realNMI}
\end{table}
Table~\ref{tbl:realNMI} summarizes the normalized mutual information between the found clustering and the clustering which is defined by the classification. We observe that the unnormalized version of \textsc{SpectACl} obtains the highest NMI on the pulsar dataset, the second-highest on the Sloan dataset, a significantly lower NMI than \textsc{RSC}, \textsc{SC} and the normalized version on the MNIST sample and a similar NMI on the SNAP dataset. The Sloan and MNIST datasets pose interesting cases. The notoriously underperforming DBSCAN obtains the highest NMI at the Sloan dataset while obtaining the lowest $F$-measure of $0.09$, where \textsc{SpectACl} obtains the highest $F$-measure of $0.52$. The datasets Pulsar and Sloan are clustered in accordance with the classes when using the density-based approach, yet the MNIST clustering corresponds more with the clustering when using the minimum-cut associated approaches.

To understand why \textsc{SpectACl} fails to return clusters in accordance with the MNIST classes, we display some exemplary points for three \textsc{SpectACl} clusters in Figure~\ref{fig:mnist}. On the one hand, some identified clusters indeed contain only one cipher, e.g., there are clusters for zeros, ones and sixes exclusively. On the other hand, the clusters depicted in the figure contain a mixture of digits, where the digits are written in a cursive style or rather straight. Hence, instead of identifying the numbers, \textsc{SpectACl} identifies clusters of handwriting style, which is interesting information discovered from the dataset, albeit not the information the MNIST task asks for. To be fair, it is far from impossible that a similar rationale can be given for the clusters \textsc{SpectACl}'s competitors find on the Pulsar or Sloan dataset; we lack the astrophysical knowledge required to similarly assess those results.
\begin{figure}[t]
  \centering
  \includegraphics[width=.95\columnwidth]{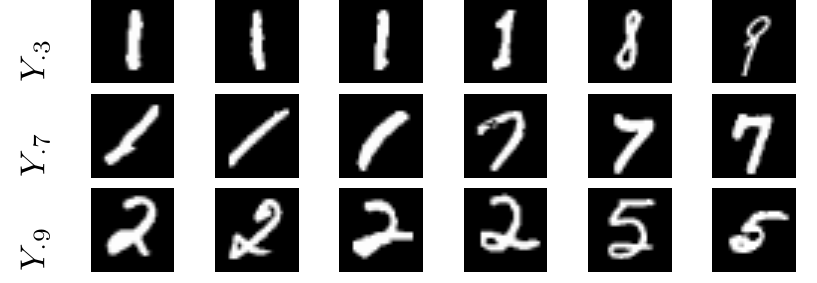}
  \caption{\textsc{SpectACl} clusters individual handwriting styles instead of cipher shapes.\label{fig:mnist}}
\end{figure}
\section{Conclusions}
We introduce \textsc{SpectACl} (\emph{Spectral Averagely-dense Clustering}): a new clustering method that combines benefits from Spectral Clustering and the density-based DBSCAN algorithm, while avoiding some of their drawbacks. 

By computing the spectrum of the weighted adjacency matrix, \textsc{SpectACl} automatically determines the appropriate density for each cluster.  This eliminates the specification of the $minpts$ parameter which is required in DBSCAN, and as we have seen in Figure \ref{fig:intro}, this specification is a serious hurdle for a DBSCAN user to overcome. On two concentric circles with the same number of observations (and hence with different densities), a DBSCAN run with $minpts=25$ lumps all observations into one big cluster.  When increasing $minpts$ by one, DBSCAN jumps to four clusters, none of which are appropriate. \textsc{SpectACl}, conversely, can natively handle these nonconvex clusters with varying densities.

In both Spectral Clustering and \textsc{SpectACl}, the final step is to postprocess intermediate results with $k$-means. Whether this choice is appropriate for Spectral Clustering remains open to speculation.  However, from the objective function of \textsc{SpectACl}, we derive an upper bound through the eigenvector decomposition, whose optimization we show to be equal to $k$-means optimization.  Hence, for \textsc{SpectACl}, we demonstrate that this choice for $k$-means postprocessing is mathematically fundamentally sound.

In comparative experiments, competing with DBSCAN, Spectral Clustering, and Robust Spectral Clustering, we find on synthetic data that the unnormalized version of \textsc{SpectACl} is the most robust to noise (cf.\@ Figure~\ref{fig:noisePlot}), and finds the appropriate cluster structure in three scenarios while the competitors all fail to do so at least once (cf.\@ Figure~\ref{fig:synthViz}).  On real-life data, \textsc{SpectACl} outperforms the competition on the inherently noisy Pulsar dataset and the Sloan dataset, it performs similarly to Spectral Clustering on the SNAP dataset, and it is defeated by both Spectral Clustering and Robust Spectral Clustering on the MNIST dataset (cf.\@ Table~\ref{tbl:realNMI}). The latter observation is explained when looking at the resulting clusters: as Figure~\ref{fig:mnist} illustrates, \textsc{SpectACl} focuses on clusters representing handwriting style rather than clusters representing ciphers, which is not unreasonable behavior for an unsupervised method such as clustering; this uncovered information is merely not reflected in the NMI scores displayed in the table.

Figure~\ref{fig:paramPlot} provides evidence that \textsc{SpectACl} is robust with respect to its parameter settings.  Hence, in addition to its solid foundation in theory and good empirical performance on data, \textsc{SpectACl} provides a clustering method that is easy to use in practice.  Hence, \textsc{SpectACl} is an ideal candidate for outreach beyond data mining experts.
\section{Acknowledgments}
Part of the work on this paper has been supported by Deutsche Forschungsgemeinschaft (DFG) within the Collaborative Research Center SFB 876 ``Providing Information by Resource-Constrained Analysis'', project C1 \url{http://sfb876.tu-dortmund.de}.

\bibliographystyle{aaai}

\begin{thebibliography}{}

\bibitem[\protect\citeauthoryear{Bauckhage}{2015}]{2015Bauckhage}
Bauckhage, C.
\newblock 2015.
\newblock K-means clustering is matrix factorization.
\newblock {\em arXiv preprint arXiv:1512.07548}.

\bibitem[\protect\citeauthoryear{Bojchevski, Matkovic, and
  G\"{u}nnemann}{2017}]{2017Bojchevski}
Bojchevski, A.; Matkovic, Y.; and G\"{u}nnemann, S.
\newblock 2017.
\newblock Robust spectral clustering for noisy data: Modeling sparse
  corruptions improves latent embeddings.
\newblock In {\em Proc.\@ KDD},  737--746.

\bibitem[\protect\citeauthoryear{Chung}{1997}]{1997Chung}
Chung, F.~R.
\newblock 1997.
\newblock {\em Spectral graph theory}.
\newblock Number~92. American Mathematical Soc.

\bibitem[\protect\citeauthoryear{Collatz}{1978}]{1978Collatz}
Collatz, L.
\newblock 1978.
\newblock Spektren periodischer graphen.
\newblock {\em Results in Mathematics} 1(1-2):42--53.

\bibitem[\protect\citeauthoryear{Dhillon, Guan, and Kulis}{2004}]{2004Dhillon}
Dhillon, I.~S.; Guan, Y.; and Kulis, B.
\newblock 2004.
\newblock Kernel k-means: Spectral clustering and normalized cuts.
\newblock In {\em Proc.\@ KDD},  551--556.

\bibitem[\protect\citeauthoryear{Ding \bgroup et al\mbox.\egroup
  }{2006}]{2006Ding}
Ding, C.; Li, T.; Peng, W.; and Park, H.
\newblock 2006.
\newblock Orthogonal nonnegative matrix t-factorizations for clustering.
\newblock In {\em Proc.\@ KDD},  126--135.

\bibitem[\protect\citeauthoryear{Ding, He, and Simon}{2005}]{2005Ding}
Ding, C.; He, X.; and Simon, H.~D.
\newblock 2005.
\newblock On the equivalence of nonnegative matrix factorization and spectral
  clustering.
\newblock In {\em Proc.\@ SDM},  606--610.

\bibitem[\protect\citeauthoryear{Driver and Kroeber}{1932}]{1932Driver}
Driver, H., and Kroeber, A.
\newblock 1932.
\newblock Quantitative expression of cultural relationships.
\newblock {\em University of California Publications in American Archaeology
  and Ethnology} 31:211--256.

\bibitem[\protect\citeauthoryear{Ester \bgroup et al\mbox.\egroup
  }{1996}]{1996Ester}
Ester, M.; Kriegel, H.-P.; Sander, J.; and Xu, X.
\newblock 1996.
\newblock A density-based algorithm for discovering clusters in large spatial
  databases with noise.
\newblock In {\em Proc.\@ KDD},  226--231.

\bibitem[\protect\citeauthoryear{Fan}{1949}]{1949Fan}
Fan, K.
\newblock 1949.
\newblock On a theorem of {W}eyl concerning eigenvalues of linear
  transformations i.
\newblock {\em Proceedings of the National Academy of Sciences}
  35(11):652--655.

\bibitem[\protect\citeauthoryear{Guattery and Miller}{1998}]{1998Guattery}
Guattery, S., and Miller, G.~L.
\newblock 1998.
\newblock On the quality of spectral separators.
\newblock {\em SIAM Journal on Matrix Analysis and Applications}
  19(3):701--719.

\bibitem[\protect\citeauthoryear{Hagen and Kahng}{1992}]{1992Hagen}
Hagen, L., and Kahng, A.~B.
\newblock 1992.
\newblock New spectral methods for ratio cut partitioning and clustering.
\newblock {\em IEEE transactions on computer-aided design of integrated
  circuits and systems} 11(9):1074--1085.

\bibitem[\protect\citeauthoryear{Halko, Martinsson, and
  Tropp}{2011}]{2011Halko}
Halko, N.; Martinsson, P.-G.; and Tropp, J.~A.
\newblock 2011.
\newblock Finding structure with randomness: Probabilistic algorithms for
  constructing approximate matrix decompositions.
\newblock {\em SIAM review} 53(2):217--288.

\bibitem[\protect\citeauthoryear{Hou, Gao, and Li}{2016}]{2016Hou}
Hou, J.; Gao, H.; and Li, X.
\newblock 2016.
\newblock D{S}ets-{DBSCAN}: A parameter-free clustering algorithm.
\newblock {\em IEEE Transactions on Image Processing} 25(7):3182--3193.

\bibitem[\protect\citeauthoryear{Kang \bgroup et al\mbox.\egroup
  }{2018}]{2018Kang}
Kang, Z.; Peng, C.; Cheng, Q.; and Xu, Z.
\newblock 2018.
\newblock Unified spectral clustering with optimal graph.
\newblock In {\em Proc.\@ AAAI}.

\bibitem[\protect\citeauthoryear{Klimek}{1935}]{1935Klimek}
Klimek, S.
\newblock 1935.
\newblock Culture element distributions: I, the structure of {C}alifornian
  {I}ndian culture.
\newblock {\em University of California Publications in American Archaeology
  and Ethnology} 37:1--70.

\bibitem[\protect\citeauthoryear{Kuhn}{1955}]{1955Kuhn}
Kuhn, H.~W.
\newblock 1955.
\newblock The {H}ungarian method for the assignment problem.
\newblock {\em Naval research logistics quarterly} 2(1-2):83--97.

\bibitem[\protect\citeauthoryear{Lecun \bgroup et al\mbox.\egroup
  }{1998}]{1998Lecun}
Lecun, Y.; Bottou, L.; Bengio, Y.; and Haffner, P.
\newblock 1998.
\newblock Gradient-based learning applied to document recognition.
\newblock {\em Proceedings of the IEEE} 86(11):2278--2324.

\bibitem[\protect\citeauthoryear{Leskovec and Krevl}{2014}]{2014Leskovec}
Leskovec, J., and Krevl, A.
\newblock 2014.
\newblock {SNAP Datasets}: {Stanford} large network dataset collection.
\newblock \url{http://snap.stanford.edu/data}.

\bibitem[\protect\citeauthoryear{Lloyd}{1982}]{1982Lloyd}
Lloyd, S.~P.
\newblock 1982.
\newblock Least squares quantization in {PCM}.
\newblock {\em IEEE Transactions on Information Theory} 28(2):129--137.

\bibitem[\protect\citeauthoryear{Nie \bgroup et al\mbox.\egroup
  }{2017}]{2017Nie}
Nie, F.; Wang, X.; Deng, C.; and Huang, H.
\newblock 2017.
\newblock Learning a structured optimal bipartite graph for co-clustering.
\newblock In {\em Advances in Neural Information Processing Systems},
  4132--4141.

\bibitem[\protect\citeauthoryear{Saad}{2011}]{2011Saad}
Saad, Y.
\newblock 2011.
\newblock {\em Numerical methods for large eigenvalue problems: revised
  edition}, volume~66.
\newblock SIAM.

\bibitem[\protect\citeauthoryear{Sch\"olkopf, Smola, and
  M\"uller}{1998}]{1998Schoelkopf}
Sch\"olkopf, B.; Smola, A.; and M\"uller, K.-R.
\newblock 1998.
\newblock Nonlinear component analysis as a kernel eigenvalue problem.
\newblock {\em Neural Computation} 10(5):1299--1319.

\bibitem[\protect\citeauthoryear{Tryon}{1939}]{1939Tryon}
Tryon, R.~C.
\newblock 1939.
\newblock {\em Cluster Analysis}.
\newblock Ann Arbor, Michigan: Edwards Brothers.

\bibitem[\protect\citeauthoryear{Vavasis}{2009}]{vavasis2009complexity}
Vavasis, S.~A.
\newblock 2009.
\newblock On the complexity of nonnegative matrix factorization.
\newblock {\em SIAM Journal on Optimization} 20(3):1364--1377.

\bibitem[\protect\citeauthoryear{Von~Luxburg}{2007}]{2007Luxburg}
Von~Luxburg, U.
\newblock 2007.
\newblock A tutorial on spectral clustering.
\newblock {\em Statistics and computing} 17(4):395--416.

\end{thebibliography}

\end{document}